\documentclass[a4paper,10pt]{article}

\usepackage[utf8x]{inputenc}
\usepackage{amsmath}
\usepackage{amssymb}
\usepackage{amsthm}
\usepackage{cite}
\usepackage{pgfplots}
\usepackage{geometry}
\usepackage[authoryear, round]{natbib}
\usepackage{amsfonts}

\usepackage{mathrsfs}

\newtheorem{definition}{Definition}
\newtheorem{algorithm}{Algorithm}
\newtheorem{theorem}{Theorem}

\newtheorem{lemma}{Lemma}

\begin{document}

\title{Context-specific independence in graphical log-linear models}

\author{Henrik Nyman$^{1, \ast}$, Johan Pensar$^{1}$, Timo Koski$^{3}$, Jukka Corander$^{2}$ \\
$^{1}$Department of Mathematics, \AA bo Akademi University, Finland \\
$^{2}$Department of Mathematics and statistics, University of Helsinki, Finland \\
$^{3}$Department of Mathematics \\ KTH Royal Institute of Technology, Stockholm, Sweden \\
$^{\ast}$Corresponding author, Email: hennyman@abo.fi}

\date{}

\maketitle

\begin{abstract}
Log-linear models are the popular workhorses of analyzing contingency tables. A log-linear parameterization of an interaction model can be more expressive than a direct parameterization based on probabilities, leading to a powerful way of defining restrictions derived from marginal, conditional and context-specific independence. However, parameter estimation is often simpler under a direct parameterization, provided that the model enjoys certain decomposability properties. Here we introduce a cyclical projection algorithm for obtaining maximum likelihood estimates of log-linear parameters under an arbitrary context-specific graphical log-linear model, which needs not satisfy criteria of decomposability. We illustrate that lifting the restriction of decomposability makes the models more expressive, such that additional context-specific independencies embedded in real data can be identified. It is also shown how a context-specific graphical model can correspond to a non-hierarchical log-linear parameterization with a concise interpretation. This observation can pave way to further development of non-hierarchical log-linear models, which have been largely neglected due to their believed lack of interpretability.   
\end{abstract}


\section{Introduction}
Log-linear models for contingency tables have enjoyed a wide popularity since their introduction in the 1970's, enabling a comprehensive approach to testing hypotheses of marginal and conditional independence, as well as more detailed global scrutiny of inter-dependencies within a set of discrete variables \citep{Whittaker90, Lauritzen96}. Graphical models have received most of the attention within the class of log-linear models, which is unsurprising given their interpretability and relative ease of model fitting. However, several other dependency structures with a log-linear representation have also been considered, such as hierarchical \citep{Lauritzen96}, pairwise interaction \citep{Whittaker90}, split \citep{Hojsgaard03}, labeled \citep{Corander03a}, and context-specific interaction models \citep{Eriksen99,Hojsgaard04}.

Recently, \citet{Nyman14a} introduced a class of stratified graphical models (SGMs), where strata are defined locally in the outcome space such that a specific pair of variables are independent in the context defined by a combination of values of the joint neighbors of the two variables. This is in contrast to ordinary graphical models, where a pair of variables are always considered either conditionally independent or completely dependent given their joint neighbors. 

The work of \citet{Nyman14a} generalizes the results on labeled graphical models, introduced in \citet{Corander03a}. To be able to obtain an analytical expression for Bayesian model scoring of SGMs, \citet{Nyman14a} restricted their attention to a class of decomposable models under a direct parameterization of the probabilities (rather than log-linear parameterization), similar to the class of graphical models where the majority of model learning approaches have been devised under the assumption of excluding non-chordal graphs from the search space. Despite of the assumption of decomposability, the resulting model class was shown to be expressive for real data and \citet{Nyman14b} additionally illustrated that SGMs can lead to more accurate probabilistic classifiers than those based on standard graphical models.

Since the assumption of decomposability is generally made for computational convenience, rather than being motivated by data met in real applications, it is desirable to develop theory which enables fitting of context-specific graphical log-linear models irrespectively of them being decomposable or non-decomposable. Using the general estimation theory from \citet{Csiszar75}  and \citet{Rudas98}, we introduce a cyclical projection algorithm which can be used to obtain the maximum likelihood estimate for any context-specific graphical log-linear model. This result is of interest on its own, however, to also illustrate the increased expressiveness of unrestricted context-specific graphical log-linear models for real data, we combine the maximum likelihood estimation with approximate Bayesian model scoring to define a search algorithm for the optimal model for a given data set. We additionally briefly illustrate the fact that some context-specific graphical models are also non-hierarchical log-linear models. This is particularly illuminating, since non-hierarchical log-linear models have generally been avoided due to believed lack of apparent interpretation of the parameter restrictions.   

The remaining article is structured as follows. In the next section we define the basic concepts related to graphical and stratified graphical models. In Section 3, log-linear parameterization of an SGM is defined, leading to a context-specific graphical log-linear model, together with some observations concerning model identifiability. In Section 4, we introduce a projection algorithm which is proven to converge to the maximum likelihood estimate for a context-specific graphical log-linear model. In Section 5, we devise an approximate Bayesian model optimization algorithm, based on the Bayesian information criterion and a stochastic search over the model space. The algorithm is illustrated by application to real data in Section 6 and the final section provides some remarks and possibilities for future work.

\section{Stratified graphical models}
\label{secSGM}
To enable the presentation of stratified graphical models, some of the central concepts from the theory of graphical models are first introduced \citep{Nyman14a}. For a comprehensive account of the statistical and computational theory of probabilistic graphical models, see \citet{Whittaker90}, \citet{Lauritzen96}, and \citet{Koller09}. It is assumed throughout this article that all considered variables are binary. However, the introduced theory can readily be extended to finite discrete variables. 

While the terms node and variable are closely related when considering graphical models, we will strive to use the notation $X_{\delta}$ when referring to the variable associated to node $\delta$. Let $G = (\Delta,E)$, be an undirected graph, consisting of a set of nodes $\Delta$ and a set of undirected edges $E\subseteq\{\Delta \times\Delta\}$. Two nodes $\gamma$ and $\delta$ are \textit{adjacent} in a graph if $\{\gamma, \delta\}\in E$, that is an edge exists between them. A \textit{path} in a graph is a sequence of nodes such that for each successive pair within the sequence the nodes are adjacent. A \textit{cycle} is a path that starts and ends with the same node. A \textit{chord} in a cycle is an edge between two non-consecutive nodes in the cycle. Two sets of nodes $A$ and $B$ are said to be \textit{separated} by a third set of nodes $S$ if every path between nodes in $A$ and nodes in $B$ contains at least one node in $S$. A graph is defined as \textit{chordal} if all cycles found in the graph containing four or more unique nodes contains at least one chord.

For a subset of nodes $A \subseteq \Delta$, $G_{A}=(A,E_{A})$ is a subgraph of $G$, such that the nodes in $G_{A}$ are equal to $A$ and the edge set comprises those edges of the original graph for which both nodes are in $A$, i.e. $E_{A} = \{A \times A\} \cap E$. A graph is defined as \textit{complete} when all pairs of nodes in the graph are adjacent. A \textit{clique} in a graph is a set of nodes $A$ such that the subgraph $G_{A}$ is complete. A \textit{maximal clique} $C$ is a clique for which there exists no set of nodes $C^*$ such that $C \subset C^*$ and $G_{C^*}$ is also complete. The set of maximal cliques in the graph $G$ will be denoted by $\mathcal{C}(G)$. The set of \textit{separators}, $\mathcal{S}(G)$, in the chordal graph $G$ can be obtained through intersections of the maximal cliques of $G$ ordered in terms of a junction tree, see e.g. \citet{Golumbic04}.

The outcome space for the variables $X_A$, where $A \subseteq \Delta$, is denoted by $\mathcal{X}_{A}$ and an element in this space by $x_{A} \in \mathcal{X}_{A}$. Given our restriction to binary variables, the cardinality $|\mathcal{X}_{A}|$ of $\mathcal{X}_{A}$ equals $2^{|A|}$. A graphical model is defined by the pair $G=(\Delta,E)$ and the joint distribution $P_{\Delta}$ on the variables $X_{\Delta}$, such that $P_{\Delta}$ fulfills a set of restrictions induced by $G$. If there exists no path between two sets of nodes $A$ and $B$ the two sets of variables $X_A$ and $X_B$ are marginally independent, i.e. $P(X_A, X_B) = P(X_A) P(X_B)$. Similarly two sets of random variables $X_A$ and $X_B$ are conditionally independent given a third set of variables $X_S$, $P(X_A, X_B \mid X_S) = P(X_A \mid X_S) P(X_B \mid X_S)$, if $S$ separates $A$ and $B$ in $G$.

A statement of conditional independence of two variables $X_{\delta}$ and $X_{\gamma}$ given $X_S$ imposes fairly strong restrictions to the joint distribution since the condition $P(X_{\delta}, X_{\gamma} \mid X_S) = P(X_{\delta} \mid X_S) P(X_{\gamma} \mid X_S)$ must hold for any joint outcome of the variables $X_S$. The idea common to context-specific independence models is to lift some of these restrictions to achieve more flexibility in terms of model structure. Exactly which restrictions are allowed to be simultaneously lifted varies considerably over the proposed model classes.

Consider a GM with the complete graph spanning three nodes $\{1, 2, 3\}$, which specifies that there are no conditional independencies among the variables $X_1$, $X_2$, and $X_3$. However, if the probability $P(X_1=1, X_2=x_2, X_3=x_3)$ factorizes into the product $P(X_1=1) P(X_2=x_2 \mid X_1=1) P(X_3=x_3 \mid X_1=1)$ for all outcomes $x_2 \in \{0,1\}, x_3 \in \{0,1\}$, then a simplification of the joint distribution is hiding beneath the graph. This simplification can be included in the graph by adding a condition or \textit{stratum} to the edge $\{2, 3\}$ specifying where the context-specific independence $X_2 \perp X_3 \mid X_1=1$ of the two variables holds, as illustrated in Figure \ref{SGMs}a.
\begin{figure}[htb]
\begin{center}
\includegraphics[width=\textwidth]{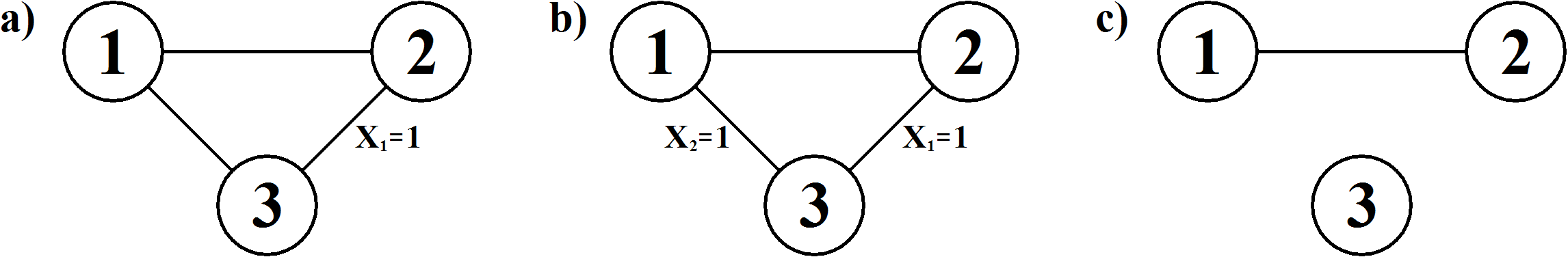}
\end{center}
\caption{Graphical representation of the dependence structures of three variables. In (a) the stratum $X_1=1$ is shown as a condition on the edge $\{2, 3\}$, in (b) the strata $X_1=1$ and $X_2=1$ are shown as conditions on the edges $\{2, 3\}$ and $\{1, 3\}$, respectively, in (c) an ordinary graph with the maximal cliques $\{1, 2\}$ and $\{3\}$.}
\label{SGMs}
\end{figure}
The following is a formal definition of what is intended by a stratum \citep{Nyman14a}.
\begin{definition} \label{stratum} Stratum. Let the pair $(G, P_{\Delta})$ be a graphical model, where $G$ is a chordal graph. For all $\{\delta,\gamma\} \in E$, let $L_{\{\delta,\gamma\}}$ denote the set of nodes adjacent to both $\delta$ and $\gamma$. For a non-empty $L_{\{\delta,\gamma\}}$, define the stratum of the edge  $\{\delta,\gamma\}$ as the subset $\mathcal{L}_{\{\delta,\gamma\}}$ of outcomes $x_{L_{\{\delta,\gamma\}}} \in \mathcal{X}_{L_{\{\delta,\gamma\}}}$ for which $X_{\delta}$ and $X_{\gamma}$ are independent given $X_{L_{\{\delta,\gamma\}}} = x_{L_{\{\delta,\gamma\}}}$, i.e. $\mathcal{L}_{\{\delta,\gamma\}} = \{ x_{L_{\{\delta,\gamma\}}} \in \mathcal{X}_{L_{\{\delta,\gamma\}}} : X_{\delta} \perp X_{\gamma} \mid X_{L_{\{\delta,\gamma\}}} =  x_{L_{\{\delta,\gamma\}}} \}$.
\end{definition}
The requirement that $G$ is chordal is necessary for the definition of a stratum to be generally applicable. Consider the graph in Figure \ref{fig:ND}, note that the graph is not chordal as it contains the chord-less cycle (1, 3, 4, 2, 1). The intended context-specific independence $X_3 \perp X_4 \mid X_5 = 1$ induced by the stratum $\mathcal{L}_{\{3, 4\}} = \{X_5=1\}$ does not hold as nodes $3$ and $4$ are connected via the path $(3, 1, 2, 4)$. By definition, no such paths are possible for chordal graphs, ensuring that given $x_{L_{\{\delta,\gamma\}}} \in \mathcal{L}_{\{\delta,\gamma\}}$ it will hold that $X_{\delta} \perp X_{\gamma} \mid X_{L_{\{\delta,\gamma\}}} =  x_{L_{\{\delta,\gamma\}}}$.
\begin{figure}[htb]
\begin{center}
\includegraphics{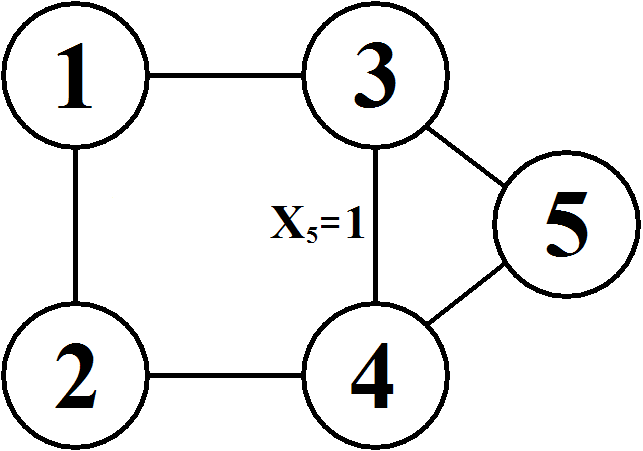}
\end{center}
\caption{Non-chordal graph resulting in the intended context-specific independence $X_3 \perp X_4 \mid X_5 = 1$ not holding.}
\label{fig:ND}
\end{figure}

The idea of context-specific independence generalizes readily to a situation where multiple strata for distinct pairs of variables are considered. Figure \ref{SGMs}b displays the complete graph for three nodes with the edges $\{2, 3\}$ and $\{1, 3\}$ associated with the strata $X_1=1$ and $X_2=1$, respectively. In addition to the context-specific independence statement present in Figure \ref{SGMs}a, here we have the simultaneous restriction that $X_1 \perp X_3 \mid X_2=1$, such that $P(X_1=x_1, X_2=1, X_3=x_3) = P(X_2=1) P(X_1=x_1 \mid X_2=1) P(X_3=x_3 \mid X_2=1)$ for all outcomes $x_1\in \{0,1\}, x_3\in\{0,1\}$. This pair of restrictions does not imply that $P(X_3=x_3) =  P(X_3=x_3 \mid X_1=1, X_2=1)$ as would be the case given the graph in Figure \ref{SGMs}c. It does, however, imply that the information contained about $X_3$ in the knowledge that $X_1=1$ and $X_2=1$ must be the same, i.e. $P(X_3=x_3 \mid X_1=1) = P(X_3=x_3 \mid X_2=1) = P(X_3=x_3 \mid X_1=1, X_2=1)$. The following definition \citep{Nyman14a} formalizes an extension to ordinary graphical models. This defined class of models allow for simultaneous context-specific independence to be represented using a set of strata, partitioning the joint outcome space of the variables $X_{\Delta}$.
\begin{definition}
Stratified graphical model (SGM). A stratified graphical model is defined by the triple $(G, L, P_{\Delta})$, where $G$ is a chordal graph termed as the underlying graph, $L$ equals the joint collection of all strata $\mathcal{L}_{\{\delta,\gamma\}}$ for the edges of $G$, and $P_{\Delta}$ is a joint distribution over the variables $X_{\Delta}$ which factorizes according to the restrictions imposed by $G$ and $L$.
\end{definition}
The pair $(G, L)$ consisting of the graph $G$ with the stratified edges (edges associated with a stratum) determined by $L$ will be referred to as a stratified graph (SG), usually denoted by $G_L$.
To be able to calculate the marginal likelihood of a dataset given an SG and perform model inference, \citet{Nyman14a} specified strict restrictions on the set of stratified edges, limiting the model space to decomposable SGs. In this paper we introduce a method which allows us to remove these restrictions while retaining the ability to perform model inference. Using the log-linear parameterization some new properties for SGMs are also introduced.

\section{SGMs and the log-linear parameterization}
\label{sec:logLin}
In this paper we use two different parameterizations. Firstly, the standard parameterization used for a categorical distribution, where each parameter $\theta_i$ in a parameter vector $\theta$ denotes the probability of a specific outcome $x_{\Delta}^{(i)} \in \mathcal{X}_{\Delta}$, i.e $P(X_{\Delta} = x_{\Delta}^{(i)}) = \theta_i$. And secondly, the log-linear parameterization \citep{Whittaker90, Lauritzen96} defined by the parameter vector $\phi$. For this parameterization, the joint distribution of the variables $X_{\Delta}$ is defined by
\[
\log P(X_{\Delta}=x_{\Delta})=\sum_{A \subseteq \Delta}\phi_{A(x_{A})},
\]
where $x_{A}$ denotes the marginal outcome of variables $X_A$ in the outcome $x_{\Delta}$. For the log-linear parameterization we have the restriction that if $x_{j}=0$ for any $j \in A$, then $\phi_{A(x_{A})}=0$ \citep{Whittaker90}. As we in this paper only consider binary variables a log-linear parameter will henceforth be denoted using the convention $\phi_{A(x_{A})} = \phi_{A}$. The reason for using the log-linear parameterization is that an SG imposes restrictions to $\phi$ in a more manageable manner than it does to $\theta$.

It holds for graphical log-linear models that if the edge $\{\delta,\gamma\}$ is not present in $G$, then all parameters $\phi_{A}$, where $\{\delta ,\gamma\} \subseteq A$, are equal to zero \citep{Whittaker90}. The restrictions imposed to the log-linear parameters by a stratum are also clearly defined.
\begin{theorem}
\label{th:restrictions}
Consider the context-specific independence $X_{\delta} \perp X_{\gamma} \mid X_{L_{\{\delta,\gamma\}}} =  x_{L_{\{\delta,\gamma\}}}$. Let $A \subseteq L_{\{\delta,\gamma\}} \cup \{\delta,\gamma\}$ be the set of variables containing the pair $\{\delta,\gamma\}$ and the set of all variables with non-zero values in $x_{L_{\{\delta,\gamma\}}}$. The parameter restrictions imposed are of the form $\sum_{B \subseteq A} \phi_{B} = 0$, where $\{\delta, \gamma\} \subseteq B$.
\end{theorem}
\begin{proof}
We start by defining the operator $\mathcal{D}(A;B) = \{ \{A \cup C \} : C \subseteq B\}$. Let $\Omega = \Delta \setminus \{L_{\{\delta, \gamma\}} \cup \{\delta, \gamma\}\}$ denote the set of nodes not in $L_{\{\delta, \gamma\}}$ or $\{\delta, \gamma\}$. Given that $X_{\delta} \perp X_{\gamma} \mid X_{L_{\{\delta,\gamma\}}} =  x_{L_{\{\delta,\gamma\}}}$ we get that
\[
\begin{split}
& \frac{P(X_{\delta} = 0 \mid X_{\gamma} = 0, X_{L_{\{\delta, \gamma\}}} = x_{L_{\{\delta,\gamma\}}}, X_{\Omega} = x_{\Omega})}{P(X_{\delta} = 1 \mid X_{\gamma} = 0, X_{L_{\{\delta, \gamma\}}} = x_{L_{\{\delta,\gamma\}}}, X_{\Omega} = x_{\Omega})} = \\ 
& \frac{P(X_{\delta} = 0 \mid X_{\gamma} = 1, X_{L_{\{\delta, \gamma\}}} = x_{L_{\{\delta,\gamma\}}}, X_{\Omega} = x_{\Omega})}{P(X_{\delta} = 1 \mid X_{\gamma} = 1, X_{L_{\{\delta, \gamma\}}} = x_{L_{\{\delta,\gamma\}}}, X_{\Omega} = x_{\Omega})} \\
\end{split}
\]
\[
\Longleftrightarrow
\]
\begin{equation}
\label{eq:logLin}
\begin{split}
& \frac{P(X_{\delta} = 0, X_{\gamma} = 0, X_{L_{\{\delta, \gamma\}}} = x_{L_{\{\delta,\gamma\}}}, X_{\Omega} = x_{\Omega})}{P(X_{\delta} = 1, X_{\gamma} = 0, X_{L_{\{\delta, \gamma\}}} = x_{L_{\{\delta,\gamma\}}}, X_{\Omega} = x_{\Omega})} = \\ 
& \frac{P(X_{\delta} = 0, X_{\gamma} = 1, X_{L_{\{\delta, \gamma\}}} = x_{L_{\{\delta,\gamma\}}}, X_{\Omega} = x_{\Omega})}{P(X_{\delta} = 1, X_{\gamma} = 1, X_{L_{\{\delta, \gamma\}}} = x_{L_{\{\delta,\gamma\}}}, X_{\Omega} = x_{\Omega})}.
\end{split}
\end{equation}
Let $Z$ denote the set of nodes corresponding to variables with non-zero outcome in $x_{L_{\{\delta,\gamma\}}}$ or $x_{\Omega}$. 
Using the log-linear parameterization, equation \eqref{eq:logLin} results in
\[
\sum_{a \subseteq \mathcal{D}(\varnothing; Z)} \hspace{-0.3cm} \phi_a \quad - \sum_{a \subseteq \mathcal{D}(\varnothing; \delta \cup Z)} \hspace{-0.5cm} \phi_a \quad = \sum_{a \subseteq \mathcal{D}(\varnothing; \gamma \cup Z)} \hspace{-0.5cm} \phi_a \quad - \sum_{a \subseteq \mathcal{D}(\varnothing; \{\delta, \gamma \} \cup Z)} \hspace{-0.7cm} \phi_a \qquad \Rightarrow
\]
\[
\sum_{a \subseteq \mathcal{D}(\delta; Z)} \hspace{-0.3cm} \phi_a \quad = \sum_{a \subseteq \mathcal{D}(\delta; \gamma \cup Z)} \hspace{-0.3cm} \phi_a \qquad \Rightarrow \qquad \sum_{a \subseteq \mathcal{D}(\{\delta, \gamma \}; Z)} \hspace{-0.3cm} \phi_a \quad = \quad 0.
\]
However, if a node $\zeta \in \Omega$, it cannot be adjacent to both $\delta$ and $\gamma$. Consequently, any parameter $\phi_A$ such that $\{\delta, \gamma, \zeta\} \subseteq A$ is restricted to zero. Therefore, if $L_{Z}$ denotes the nodes corresponding to the variables with non-zero outcome in $x_{L_{\{\delta,\gamma\}}}$ the restriction induced by the stratum can be written
\[
\sum_{a \subseteq \mathcal{D}(\{\delta, \gamma \}; L_{Z})} \hspace{-0.5cm} \phi_a \quad = \quad 0,
\]
which corresponds to what is stated in the theorem.
\end{proof}
\noindent As an example consider the SG in Figure \ref{SGMs}a. The context-specific independence $X_{2} \perp X_{3} \mid X_1=1$ induces the log-linear parameter restriction
\[
\begin{split}
& \phi_{\varnothing} + \phi_{1} - \phi_{\varnothing} - \phi_{1} - \phi_{2} - \phi_{1,2} = \\
& \phi_{\varnothing} + \phi_{1} + \phi_{3} + \phi_{1,3} - \phi_{\varnothing} - \phi_{1} - \phi_{2} - \phi_{3} - \phi_{1,2} - \phi_{1,3} - \phi_{2,3} - \phi_{1,2,3}  \Rightarrow \\
& \phi_{2,3} + \phi_{1,2,3} = 0.
\end{split}
\]

In the definition of a stratum on the edge $\{\delta, \gamma\}$, the variables that determine the stratum correspond to the nodes that are adjacent to both $\delta$ and $\gamma$. This is a natural definition rather than an invented restriction.
\begin{theorem}
\label{th:adjacent}
Only the variables corresponding to nodes adjacent to both $\delta$ and $\gamma$ may define a context-specific independence between $X_{\delta}$ and $X_{\gamma}$.
\end{theorem}
\begin{proof}
The proof of this theorem follows from Theorem \ref{th:restrictions}. We assume that a variable $X_{\zeta}$, such that the node $\zeta$ is not adjacent to both $\delta$ and $\gamma$, is included when defining the context-specific independence $X_{\delta} \perp X_{\gamma} \mid X_{L_{\{\delta,\gamma\}}} =  x_{L_{\{\delta,\gamma\}}}, X_{\zeta} = x_{\zeta}$. If $x_{\zeta} = 0$, we would get the same restriction as by not including $X_{\zeta}$ in the conditioning set
\begin{equation}
\label{eq:same1}
\sum_{a \subseteq \mathcal{D}(\{\delta, \gamma \}; L_{Z})} \hspace{-0.5cm} \phi_a \quad = \quad 0.
\end{equation}
If $x_{\zeta} \neq 0$ we get the restriction
\begin{equation}
\label{eq:same2}
\sum_{a \subseteq \mathcal{D}(\{\delta, \gamma \}; \{L_{Z}, \zeta\})} \hspace{-0.5cm} \phi_a \quad = \quad 0,
\end{equation}
again we know from the underlying graph that any parameter $\phi_A$ such that $\{\delta, \gamma, \zeta\} \subseteq A$ is restricted to zero, resulting in \eqref{eq:same1} and \eqref{eq:same2} being equivalent restrictions.
\end{proof}
\noindent As an example consider the graphs in Figure \ref{fig:logLin}a and \ref{fig:logLin}b. 
\begin{figure}[htb]
\begin{center}
\includegraphics[width=\textwidth]{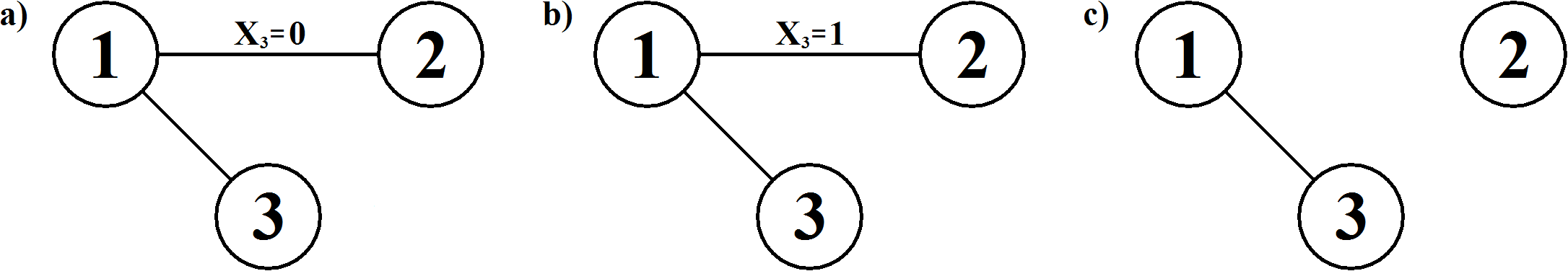}
\end{center}
\caption{Graphs in a) and b) include improper strata, which if allowed would lead to the same parameter restrictions as the graph in c).}
\label{fig:logLin}
\end{figure}
Here $X_3$ determines the stratum of the edge $\{1,2\}$, although node $2$ and node $3$ are non-adjacent. The underlying graph establishes that $X_2 \perp X_3 \mid X_1$. However, given the proposed stratum $X_3$ can indirectly affect $X_2$ by determining whether or not $X_1$ and $X_2$ are dependent, which is an obvious contradiction. The underlying graph induces the parameter restrictions $\phi_{2, 3} = \phi_{1, 2, 3}=0$. The stratum included in Figure \ref{fig:logLin}a results in the restriction $\phi_{1, 2} = 0$, while the stratum included in Figure \ref{fig:logLin}b results in the restriction $\phi_{1, 2} + \phi_ {1, 2, 3} = 0$, i.e. $\phi_{1, 2} = 0$. This means that the graphs in Figure \ref{fig:logLin} would all induce the same restrictions, $\phi_{1, 2} = \phi_{2, 3} = \phi_{1, 2, 3}=0$. Note that this example is a special case of Theorem \ref{th:adjacent} as $L_{\{1, 2\}} = \varnothing$.

\citet{Whittaker90} termed a log-linear model as hierarchical if, whenever a parameter $\phi_a = 0$ then $\phi_t=0$ for all $a \subseteq t$. \citet[p.~209]{Whittaker90} further states that ``A non-hierarchical model is not necessarily uninteresting; it is just that the focus of interest is something other than independence". This statement does not apply to SGMs, as shown in Theorem \ref{th:whittaker}.
\begin{theorem}
\label{th:whittaker}
Some, but not all, SGMs are non-hierarchical models. 
\end{theorem}
\begin{proof}
This theorem can be proved using two simple examples. First, consider an SG containing no strata, the parameter restrictions of this model will equal those of an ordinary graphical log-linear model which is a hierarchical model. Now consider the SG attained by replacing the stratum $X_{2} \perp X_{3} \mid X_1=1$ in Figure \ref{SGMs}a with $X_{2} \perp X_{3} \mid X_1=0$, this leads to the single parameter restriction $\phi_{2, 3} = 0$. As the parameter $\phi_{1, 2, 3}$ is unrestricted the model is non-hierarchical.
\end{proof}

\section{Parameter estimation for log-linear SGMs}

Let $\Theta_{G}$ denote the set of distributions satisfying the restrictions imposed by the chordal graph $G$. \citet[p.~91]{Lauritzen96} showed that given an observed distribution $P$ the maximum likelihood (ML) projection to $\Theta_{G}$, resulting in the distribution $\hat{P}$, is obtained by setting
\begin{equation}
\hat{\theta}_{i}=\hat{P}(x_{\Delta}^{(i)})=\frac{\prod_{C \in \mathcal{C}(G)}P(x_{C}^{(i)})}{\prod_{S \in \mathcal{S}(G)}P(x_{S}^{(i)})}, \ i=1, \ldots, |\mathcal{X}_{\Delta}|.
\label{eq:cliSep}
\end{equation}
Given the following definition of the Kullback-Leibler (KL) divergence
\[
D_{KL}(P | \hat{P}) = \sum_{x_{\Delta} \in \mathcal{X}_{\Delta}} \log \left( \frac{P(x_{\Delta})}{\hat{P}(x_{\Delta})} \right) P(x_{\Delta}),
\] 
\citet[p.~238]{Lauritzen96} also showed that the ML projection corresponds to finding the distribution that minimizes $D_{KL}$ in the second argument, i.e.
\[
\hat{P}=\arg \min_{Q \in \Theta_{G}} D_{KL}(P, Q).
\]
We shall later refer to the minimum discrimination information (MDI) projection, resulting in a distribution $\hat{R}$ given a distribution $R$. The MDI projection is also defined through the KL divergence but in this case as the distribution that minimizes $D_{KL}$ in the first argument, i.e.
\[
\hat{R}=\arg \min_{Q \in \Theta_{G}} D_{KL}(Q, R).
\]

The ML projection for imposing a single context-specific independence on a distribution can also be written in closed form. Consider an outcome $x_{L_{\{\delta,\gamma\}}} \in \mathcal{L}_{\{\delta,\gamma\}}$, which implies the context-specific independence $X_{\delta} \perp X_{\gamma} \mid X_{L_{\{\delta,\gamma\}}} = x_{L_{\{\delta, \gamma\}}}$. If we by $\Omega = \Delta \setminus \{L_{\{\delta, \gamma\}} \cup \{\delta, \gamma\}\}$ denote all nodes not in $L_{\{\delta, \gamma\}}$ or $\{\delta, \gamma\}$, the probability 
\[
P(X_{\Delta} = x_{\Delta}) = P(X_{L_{\{\delta,\gamma\}}} = x_{L_{\{\delta,\gamma\}}}, X_{\Omega} = x_{\Omega}, X_{\delta} = x_{\delta}, X_{\gamma} = x_{\gamma})
\]
for any $x_{\Omega} \in \mathcal{X}_{\Omega}$ can be factorized as
\[
\begin{split}
& P(X_{L_{\{\delta,\gamma\}}} = x_{L_{\{\delta,\gamma\}}}, X_{\Omega} = x_{\Omega}) 
P(X_{\delta} = x_{\delta} \mid X_{L_{\{\delta, \gamma\}}} = x_{L_{\{\delta, \gamma\}}}, X_{\Omega} = x_{\Omega}) \\
& P(X_{\gamma} = x_{\gamma} \mid X_{L_{\{\delta,\gamma\}}} = x_{L_{\{\delta,\gamma\}}}, X_{\Omega} = x_{\Omega}).
\end{split}
\]
Using the following abbreviated notation for $\theta$ (and correspondingly for $\hat{\theta}$)
\[
\begin{split}
\theta_{0,0} &  = P(X_{L_{\{\delta, \gamma\}}} = x_{L_{\{\delta,\gamma\}}}, X_{\Omega} = x_{\Omega}, X_{\delta} = 0, X_{\gamma} = 0), \\
\theta_{0,1} &  = P(X_{L_{\{\delta, \gamma\}}} = x_{L_{\{\delta,\gamma\}}}, X_{\Omega} = x_{\Omega}, X_{\delta} = 0, X_{\gamma} = 1), \\
\theta_{1,0} &  = P(X_{L_{\{\delta, \gamma\}}} = x_{L_{\{\delta,\gamma\}}}, X_{\Omega} = x_{\Omega}, X_{\delta} = 1, X_{\gamma} = 0), \\
\theta_{1,1} &  = P(X_{L_{\{\delta, \gamma\}}} = x_{L_{\{\delta,\gamma\}}}, X_{\Omega} = x_{\Omega}, X_{\delta} = 1, X_{\gamma} = 1),
\end{split}
\]
we determine the values $\hat{\theta}_{0,0}$, $\hat{\theta}_{0,1}$, $\hat{\theta}_{1,0}$, and $\hat{\theta}_{1,1}$ according to
\begin{equation}
\begin{split}
\hat{\theta}_{0,0} &  =(\theta_{0,0}+\theta_{0,1})\cdot(\theta_{0,0}+\theta
_{1,0})/(\theta_{0,0}+\theta_{0,1}+\theta_{1,0}+\theta_{1,1}),\\
\hat{\theta}_{0,1} &  =(\theta_{0,0}+\theta_{0,1})\cdot(\theta_{0,1}+\theta
_{1,1})/(\theta_{0,0}+\theta_{0,1}+\theta_{1,0}+\theta_{1,1}),\\
\hat{\theta}_{1,0} &  =(\theta_{1,0}+\theta_{1,1})\cdot(\theta_{0,0}+\theta
_{1,0})/(\theta_{0,0}+\theta_{0,1}+\theta_{1,0}+\theta_{1,1}),\\
\hat{\theta}_{1,1} &  =(\theta_{1,0}+\theta_{1,1})\cdot(\theta_{0,1}+\theta
_{1,1})/(\theta_{0,0}+\theta_{0,1}+\theta_{1,0}+\theta_{1,1}).
\end{split}
\label{eq:theta}
\end{equation}
A detailed derivation of the projection defined above is given in Appendix A. Repeating the procedure defined in \eqref{eq:theta} for all $x_{\Omega} \in \mathcal{X}_{\Omega}$ will result in the ML projection of $P$ satisfying the context-specific independence $X_{\delta} \perp X_{\gamma} \mid X_{L_{\{\delta,\gamma\}}} = x_{L_{\{\delta, \gamma\}}}$.

\subsection{Maximum likelihood estimation for SGs}
By cyclically repeating the projections according to \eqref{eq:cliSep} and \eqref{eq:theta} for all instances found in the set of strata $L$ until convergence is achieved, the resulting parameter vector will be the maximum likelihood estimate that simultaneously satisfies all the restrictions imposed by $G_{L}$. In order to prove this we first need to define the following family of probability distributions.

\begin{definition}
Let $X_{\delta}$ and $X_{\gamma}$ be two variables in $X_{\Delta}$, $X_A$ a subset of $X_{\Delta \setminus \{\delta, \gamma\}}$ and $X_{\Omega} = X_{\Delta \setminus \{A \cup \{\delta, \gamma\}\}}$. $\mathcal{F}_{\delta, \gamma}(X_A=x_A, Q)$, where $Q$ is an arbitrary probability distribution, is defined as the set of probability distributions for which the following properties hold for all possible values $x_{\delta}$, $x_{\gamma}$ and $x_\Omega$.
\[
\begin{split}
&\mathcal{F}_{\delta, \gamma}(X_A=x_A, Q) = \\
& \{P: P(X_A=x_A, X_{\Omega}=x_{\Omega}) = Q(X_A=x_A, X_{\Omega}=x_{\Omega})\} \ \cap \\
&\{P: P(X_{\delta}=x_{\delta} | X_A=x_A, X_{\Omega}=x_{\Omega}) = Q(X_{\delta}=x_{\delta} | X_A=x_A, X_{\Omega}=x_{\Omega}) \} \ \cap \\
&\{P: P(X_{\gamma}=x_{\gamma} | X_A=x_A, X_{\Omega}=x_{\Omega}) = Q(X_{\gamma}=x_{\gamma} | X_A=x_A, X_{\Omega}=x_{\Omega}) \} \ \cap \\
&\{P: P(X_{\Delta}=y_{\Delta})=Q(X_{\Delta}=y_{\Delta}) \text{, when $y_{\Delta}$ is a outcome where $x_A \neq y_A$} \}.
\end{split}
\]
\end{definition}
A set of probability distributions $\mathscr{C}$ is defined as a \textit{linear set} if $P_1 \in \mathscr{C}$ and $P_2 \in \mathscr{C}$ results in $\alpha P_1 + (1 - \alpha)P_2$ also belonging to $\mathscr{C}$ for every real $\alpha$ for which it is a probability distribution \citep{Csiszar75}. 
\begin{lemma}
$\mathcal{F}_{\delta, \gamma}(X_A=x_A, Q)$ constitutes a linear set.
\end{lemma}
\begin{proof}
Let $P_1$ and $P_2$ be two probability distributions in $\mathcal{F}_{\delta, \gamma}(X_A=x_A, Q)$, we then need to prove that $P^*=\alpha P_1 + (1 - \alpha)P_2$ also belongs to $\mathcal{F}_{\delta, \gamma}(X_A=x_A, Q)$. It is trivial to show that $P^*(X_A=x_A, X_{\Omega}=x_{\Omega}) = Q(X_A=x_A, X_{\Omega}=x_{\Omega})$ and that $P^*(X_{\Delta}=y_{\Delta}) = Q(X_{\Delta}=y_{\Delta})$, when $y_{\Delta}$ is a outcome where $x_A \neq y_A$. The non-trivial part consists of showing that $P^*(X_{\delta}=x_{\delta} | X_A=x_A, X_{\Omega}=x_{\Omega}) = Q(X_{\delta}=x_{\delta} | X_A=x_A, X_{\Omega}=x_{\Omega})$. We start from the fact that
\[
\begin{split}
Q(X_{\delta}=x_{\delta} | X_A=x_A, X_{\Omega}=x_{\Omega}) 
&= P_1(X_{\delta}=x_{\delta} | X_A=x_A, X_{\Omega}=x_{\Omega}) \\ 
&= P_2(X_{\delta}=x_{\delta} | X_A=x_A, X_{\Omega}=x_{\Omega}),
\end{split}
\]
which implicates that
\[
\frac{P_1(X_{\delta}=x_{\delta}, X_A=x_A, X_{\Omega}=x_{\Omega})}{P_1(X_A=x_A, X_{\Omega}=x_{\Omega})} = \frac{P_2(X_{\delta}=x_{\delta}, X_A=x_A, X_{\Omega}=x_{\Omega})}{P_2(X_A=x_A, X_{\Omega}=x_{\Omega})}.
\]
From the definition of $\mathcal{F}_{\delta, \gamma}(X_A=x_A, Q)$ we know that
\[
P_1(X_A=x_A, X_{\Omega}=x_{\Omega}) = P_2(X_A=x_A, X_{\Omega}=x_{\Omega}),
\]
and can therefore deduce that
\[
\begin{split}
Q(X_{\delta}=x_{\delta}, X_A=x_A, X_{\Omega}=x_{\Omega}) &= P_1(X_{\delta}=x_{\delta}, X_A=x_A, X_{\Omega}=x_{\Omega}) \\
&= P_2(X_{\delta}=x_{\delta}, X_A=x_A, X_{\Omega}=x_{\Omega}).
\end{split}
\]
For $P^*$ this means that
\[
\begin{split}
&P^*(X_{\delta}=x_{\delta} | X_A=x_A, X_{\Omega}=x_{\Omega}) = \\
&\frac{\alpha P_1(X_{\delta}=x_{\delta}, X_A=x_A, X_{\Omega}=x_{\Omega}) + (1-\alpha) P_2(X_{\delta}=x_{\delta}, X_A=x_A, X_{\Omega}=x_{\Omega})}
{\alpha P_1(X_A=x_A, X_{\Omega}=x_{\Omega})+ (1-\alpha) P_2(X_A=x_A, X_{\Omega}=x_{\Omega})} = \\
&\frac{\alpha Q(X_{\delta}=x_{\delta}, X_A=x_A, X_{\Omega}=x_{\Omega}) + (1-\alpha) Q(X_{\delta}=x_{\delta}, X_A=x_A, X_{\Omega}=x_{\Omega})}
{\alpha Q(X_A=x_A, X_{\Omega}=x_{\Omega})+ (1-\alpha) Q(X_A=x_A, X_{\Omega}=x_{\Omega})} = \\
&\frac{Q(X_{\delta}=x_{\delta}, X_A=x_A, X_{\Omega}=x_{\Omega})}
{Q(X_A=x_A, X_{\Omega}=x_{\Omega})}  = Q(X_{\delta}=x_{\delta} | X_A=x_A, X_{\Omega}=x_{\Omega}).
\end{split}
\]
Of course the same reasoning can be used to show that $P^*(X_{\gamma}=x_{\gamma} | X_A=x_A, X_{\Omega}=x_{\Omega}) = Q(X_{\gamma}=x_{\gamma} | X_A=x_A, X_{\Omega}=x_{\Omega})$, which concludes the proof.
\end{proof}

\begin{definition}
The log-linear model $LL_{\delta, \gamma}(X_A=x_A)$ is defined as the set of probability distributions which satisfy the condition $X_{\delta} \perp X_{\gamma} \mid X_{A} =  x_{A}$.
\end{definition}
It is easy to see that for any probability distribution $Q$, the sets $LL_{\delta, \gamma}(X_A=x_A)$ and $\mathcal{F}_{\delta, \gamma}(X_A=x_A, Q)$ can have at most one common distribution, denoted by $R$. It is also evident, using the same reasoning as in Appendix A, that $R$ is the result of the ML projection of any distribution in $\mathcal{F}_{\delta, \gamma}(X_A=x_A, Q)$ to $LL_{\delta, \gamma}(X_A=x_A)$. We are now ready to prove the main theorem.

\begin{theorem}
Cyclically projecting the observed distribution, $P_0$, in accordance with the procedures defined in \eqref{eq:cliSep} and \eqref{eq:theta} until convergence is achieved, will result in the maximum likelihood estimate, $\hat{P}$, which simultaneously satisfies all the restrictions imposed by a given SG, $G_L=(G, L)$.
\end{theorem}

\begin{proof}
This proof uses the results found in \cite{Rudas98}, with Theorem 2 of that paper being of paramount importance. An essential part of the proof is the so called Pythagorean identity for discrimination information, see for instance \cite{Rudas98}, which states that if $S$ belongs to a linear set and $R$ is the MDI projection of a distribution $T$ onto this set, then $D_{KL}(S, T) = D_{KL}(S, R) + D_{KL}(R, T)$.

Let $m$ denote the number of context-specific independencies in $L$, i.e. the total number of instances in all strata included in $L$. Further, let $P_l$ be the distribution attained when projecting the distribution $P_{l-1}$ according to the $l$th context-specific independence, say $X_{\delta} \perp X_{\gamma} \mid X_{L_{\{\delta,\gamma\}}} =  x_{L_{\{\delta,\gamma\}}}$, in $L$. It then holds that
\[
P_l = LL_{\delta, \gamma}(X_{L_{\{\delta,\gamma\}}} =  x_{L_{\{\delta,\gamma\}}}) \cap \mathcal{F}_{\delta, \gamma}(X_{L_{\{\delta,\gamma\}}} =  x_{L_{\{\delta,\gamma\}}}, P_{l-1}).
\]
$P_l$ is also the MDI projection of any distribution in $LL_{\delta, \gamma}(X_{L_{\{\delta,\gamma\}}} =  x_{L_{\{\delta,\gamma\}}})$ to $\mathcal{F}_{\delta, \gamma}(X_{L_{\{\delta,\gamma\}}} =  x_{L_{\{\delta,\gamma\}}}, P_{l-1})$. \cite{Rudas98} makes this statement without providing any further comment, but as it is not self-evident we have chosen to include a proof. In order to do this we turn to \citet[Theorem 1]{Csiszar03}. This theorem states that for a log-convex set $\mathcal{T}$, which $LL_{\delta, \gamma}(X_{L_{\{\delta,\gamma\}}} =  x_{L_{\{\delta,\gamma\}}})$ constitutes as it defines an exponential family, the ML projection, denoted by $R$, of an arbitrary distribution $S$ to $\mathcal{T}$ is the unique distribution that satisfies
\[
D_{KL}(S, T) \geq \min_{A \in \mathcal{T}} D_{KL}(S, A) + D_{KL}(R, T), \quad T \in \mathcal{T}.
\] 
In our case, as $\hat{P} \in LL_{\delta, \gamma}(X_{L_{\{\delta,\gamma\}}} =  x_{L_{\{\delta,\gamma\}}})$ and $P_l$ is the ML projection of any distribution $S$ in $\mathcal{F}_{\delta, \gamma}(X_{L_{\{\delta,\gamma\}}} =  x_{L_{\{\delta,\gamma\}}}, P_{l-1})$ to $LL_{\delta, \gamma}(X_{L_{\{\delta,\gamma\}}} =  x_{L_{\{\delta,\gamma\}}})$ it holds that
\[
D_{KL}(S, \hat{P}) \geq D_{KL}(S, P_l) + D_{KL}(P_l, \hat{P}) , \ S \in \mathcal{F}_{\delta, \gamma}(X_{L_{\{\delta,\gamma\}}} =  x_{L_{\{\delta,\gamma\}}}, P_{l-1}).
\] 
Which implies that $D_{KL}(S, \hat{P}) \geq D_{KL}(P_l, \hat{P})$ holds for every $S$ in $\mathcal{F}_{\delta, \gamma}(X_{L_{\{\delta,\gamma\}}} =  x_{L_{\{\delta,\gamma\}}}, P_{l-1})$ and $P_l$ is the MDI projection of any distribution in $LL_{\delta, \gamma}(X_{L_{\{\delta,\gamma\}}} =  x_{L_{\{\delta,\gamma\}}})$ to $\mathcal{F}_{\delta, \gamma}(X_{L_{\{\delta,\gamma\}}} =  x_{L_{\{\delta,\gamma\}}}, P_{l-1})$. Therefore the Pythagorean identity is applicable and we can conclude that
\begin{equation}
D_{KL}(P_{l-1}, \hat{P}) = D_{KL}(P_{l-1}, P_l) + D_{KL}(P_{l}, \hat{P}).
\label{eq:PINyman}
\end{equation}
\cite{Rudas98} showed that the Pythagorean identity is also applicable when projecting a distribution onto the set of distributions satisfying the restrictions imposed by a chordal graph. I.e. if we by $P_{m+1}$ denote the distribution that results from projecting $P_m$ to $\Theta_{G}$ according to \eqref{eq:cliSep} we get that
\begin{equation}
D_{KL}(P_{m}, \hat{P}) = D_{KL}(P_{m}, P_{m+1}) + D_{KL}(P_{m+1}, \hat{P}).
\label{eq:PIRudas}
\end{equation}
Combining \eqref{eq:PINyman} and \eqref{eq:PIRudas} and letting the projection $n+i$ be the same projection as $i$ if $n=k(m+1)$ for some value $k=0,1, \ldots$ results in
\[
D_{KL}(P_0, \hat{P}) = \sum_{l=1}^{n} D_{KL}(P_{l-1}, P_{l}) + D_{KL}(P_{n}, \hat{P}).
\]
for every $n$. The existence of $\hat{P}$ implies that for any $n$
\[
\sum_{l=1}^{n} D_{KL}(P_{l-1}, P_{l}) < \infty,
\] 
which, in turn, implies that $D_{KL}(P_{l-1}, P_l) \rightarrow 0$ as $l \rightarrow \infty$. Just as \cite{Rudas98} we can now refer to the compactness argument found in \citet[Theorem 3.2]{Csiszar75} to complete the proof.
\end{proof}

In practice we need a criterion to determine whether or not the cyclical projections have converged to $\hat{P}$. The criterion that we use terminates the projections once an entire cycle consisting of $m+1$ projections has been completed with the total sum of changes made to $\theta$ being less than a predetermined constant $\epsilon$. Using $\theta_i = (\theta_{i1}, \ldots, \theta_{ik})$ to denote the parameter after the $i$:th projection in the cycle, with $\theta_{0}$ denoting the starting value. The cyclical projections are terminated when
\[
\sum_{i=1}^{m+1} \sum_{j=1}^{k} | \theta_{ij} - \theta_{(i-1)j} | < \epsilon.
\]


\section{Bayesian Learning of SGMs}
\label{secAlgorithm}
Bayesian learning of graphical models has attained considerable interest, both in the statistical and computer science literature, see e.g. \cite{Madigan94}, \cite{Dellaportas99}, \cite{Giudici99}, \cite{Corander03b}, \cite{Giudici03}, \cite{Koivisto04}, and \cite{Corander08}. Our learning algorithms described below belong to the class of non-reversible Metropolis-Hastings algorithms, introduced by \cite{Corander06} and later further generalized and applied to learning of graphical models in \cite{Corander08}. A similar algorithm was also used in \citet{Nyman14a} for decomposable SGMs.

To allow for Bayesian learning of SGMs, we use the maximum likelihood estimation technique introduced in the previous section to derive an approximation of the marginal likelihood based on the general result for exponential families due to \cite{Schwarz78}. The approximation utilizes the Bayesian information criterion (BIC), and is written
\begin{equation}
\log P(\mathbf{X} \mid G_{L})\approx l(\mathbf{X} \mid \hat{\theta},G_{L}) - \frac{\text{dim}(\Theta \mid G_{L})}{2}\log n,
\label{MLbic}
\end{equation}
where $\hat{\theta}$ is the maximum likelihood estimate of the model parameters under the restrictions imposed by $G_{L}$, $l(\mathbf{X} \mid \hat{\theta}, G_{L})$ is the logarithm of the likelihood function corresponding to $\hat{\theta}$, and dim$(\Theta \mid G_{L})$ is the maximum number of free parameters in a distribution with the parameter restrictions induced by $G_L$. We denote the right hand side of \eqref{MLbic} by $\log S(G_L \mid \mathbf{X})$, i.e. $P(\mathbf{X} \mid G_L) \approx S(G_L \mid \mathbf{X})$.

The maximum number of free parameters in a distribution with the parameter restrictions induced by an SG can readily be calculated using the log-linear parameterization discussed in Section \ref{sec:logLin}. 

Let $\mathcal{M}$ denote the finite space of states over which the aim is to approximate the posterior distribution. In this paper we will run two separate types of searches. In one search the state space $\mathcal{M}$ will consist of all possible sets of strata for a given chordal graph. In the second search the state space will be the set of chordal graphs combined with the optimal set of strata for that graph. For $M \in \mathcal{M}$, let $Q(\cdot \mid M)$ denote the proposal function used to generate a new candidate state given the current state $M$. Under the generic conditions stated in \citet{Corander08}, the probability with which any particular candidate is picked by $Q(\cdot \mid M)$ need not be explicitly calculated or known, as long as it remains unchanged over all the iterations and the resulting chain satisfies the condition that all states can be reached from any other state in a finite number of steps. To initialize the algorithm, a starting state $M_{0}$ is determined. At iteration $t=1,2,...$ of the non-reversible algorithm, $Q(\cdot \mid M_{t-1})$ is used to generate a candidate state $M^{\ast}$, which is accepted with the probability
\begin{equation}
\min\left(  1,\frac{P(\mathbf{X} \mid M^{\ast})P(M^{\ast})}{P(\mathbf{X} \mid M_{t-1})P(M_{t-1})}\right),
\label{accept}
\end{equation}
where $P(M)$ is the prior probability assigned to $M$. The term $P(\mathbf{X} \mid M)$ denotes the marginal likelihood of the dataset $\mathbf{X}$ given $M$. If $M^{\ast}$ is accepted, we set $M_{t}=M^{\ast}$, otherwise we set $M_{t}=M_{t-1}$.

In contrast to the standard reversible Metropolis-Hastings algorithm, for this non-reversible algorithm the posterior probability $P(M \mid \mathbf{X})$ does not, in general, equal the stationary distribution of the Markov chain. Instead, a consistent approximation of $P(M \mid \mathbf{X})$ is obtained by considering the space of distinct states $\mathcal{M}_{t}$ visited by time $t$ such that
\[
\hat{P}_t(M \mid \mathbf{X}) = \frac{P(\mathbf{X} \mid M)P(M)}{\sum_{M' \in \mathcal{M}_{t}} P(\mathbf{X} \mid M')P(M')}.
\]
\citet{Corander08} proved, under rather weak conditions, that this estimator is consistent, i.e.
\[
\hat{P}_t(M \mid \mathbf{X})\overset{a.s.}{\rightarrow}P(M \mid \mathbf{X}),
\]
as $t\rightarrow\infty$. As our main interest will lie in finding the posterior optimal state, i.e.
\[
\arg\mathop{\max}_{M\in\mathcal{M}}P(M \mid \mathbf{X}).
\]
it will suffice to identify
\[
\arg\mathop{\max}_{M\in\mathcal{M}}P(\mathbf{X} \mid M)P(M).
\]

As the marginal likelihood of a dataset is not available for the models considered in this paper the approximated BIC score is used instead. The main goal of our search algorithm is to identify the stratified graph $G_L^{\text{opt}}$ optimizing $S(G_L \mid \mathbf{X} ) P(G_L)$. Under the assumption that the optimal set of strata is known for each underlying graph a Markov chain traversing the set of possible underlying graphs will eventually identify $G_L^{\text{opt}}$. Another search may be used in order to identify the optimal set of strata given the underlying graph. The proposal functions used are described in Appendix B. For the experiments conducted in the next section, in order to penalize dense graphs, the following non-uniform prior \citep{Nyman14a} is used
\[
P(G_{L}) \propto 2^{- |\Theta_G|}.
\]
Here, $|\Theta_G|$ denotes the maximum number of free parameter in a distribution satisfying the restrictions imposed by the underlying graph $G$.

\section{Illustration of SGM Learning from Data}
\label{secRes}
In this section, in order to save space, when displaying an SG we instead of writing a stratum as $(X_1=1, X_2=0)$ only write $(1, 0)$. This is possible since, given the graph, it is clear which variables define the context-specific independence when the variables are ordered by their integer labels.

The first dataset that we have investigated includes prognostic factors for coronary heart disease and can be found in \citet{Edwards85}. The data consists of 1841 observations of the six variables listed in Table \ref{tab:heart}.
\begin{table}[htb]
\begin{center}
\begin{tabular}
[c]{cll} \hline
Variable & Meaning & Range \\ \hline
$X_1$ & Smoking & No = 0, Yes = 1 \\
$X_2$ & Strenuous mental work & No = 0, Yes = 1 \\
$X_3$ & Strenuous physical work & No = 0, Yes = 1 \\
$X_4$ & Systolic blood pressure $> 140$ & No = 0, Yes = 1 \\
$X_5$ & Ratio of beta and alpha lipoproteins $> 3$ & No = 0, Yes = 1 \\
$X_6$ & Family anamnesis of coronary heart disease & No = 0, Yes = 1 \\ \hline
\end{tabular}
\end{center}
\caption{Variables in coronary heart disease data.}
\label{tab:heart}
\end{table}
In Figure \ref{fig:heart} two different SGs are displayed. The SG in Figure \ref{fig:heart}a is obtained by first conducting a search for the optimal ordinary chordal graph and then identifying the optimal set of strata for that graph. The underlying graph has the score $-6732.84$, while the SG has the score $-6721.67$. Figure \ref{fig:heart}b contains the estimated globally optimal SG, which has the score $-6713.24$. The underlying graph for this SG has the score $-6764.14$.

\begin{figure}[htb]
\begin{center}
\includegraphics[width=\textwidth]{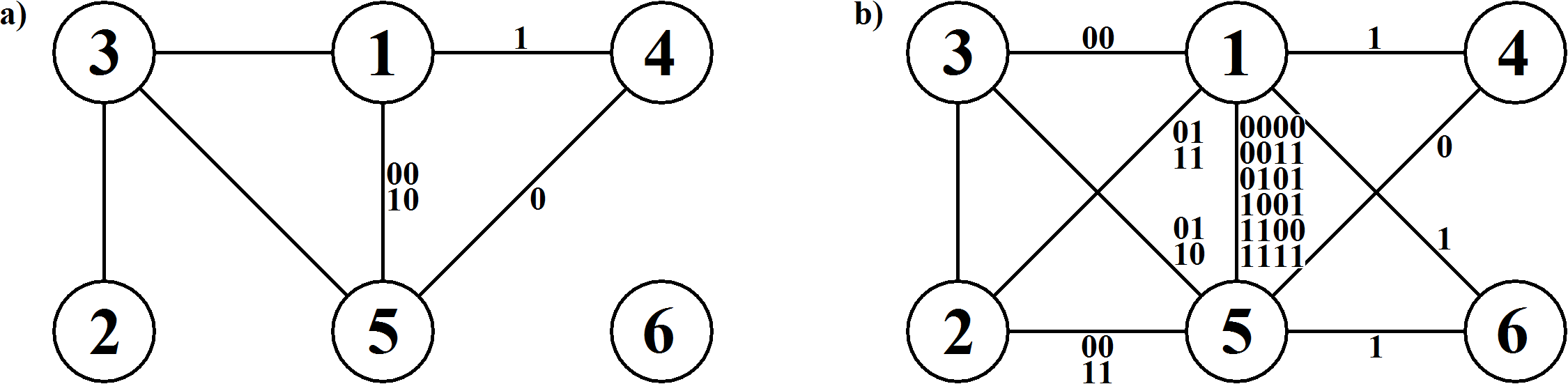}
\end{center}
\caption{Optimal SGs for heart data. In a) the optimal ordinary graph is amended with optimal strata. In b) the globally optimal SG.}
\label{fig:heart}
\end{figure}

The second dataset that we consider is derived from the answers given by 1806 candidates in the Finnish parliament elections of 2011, in a questionnaire issued by the newspaper Helsingin Sanomat \citep{HelsinginSanomat11}. The eight questions considered, represented by eight variables, are given in Appendix C. As in the previous section we present in Figure \ref{fig:HS} two different SGs, the SG resulting from first determining the optimal ordinary graph and finding the optimal set of strata for that graph and the globally optimal SG. For the globally optimal SG we, instead of displaying the exact strata, give the total number of instances included in the stratum associated with each edge. The score for the underlying graph of Figure \ref{fig:HS}a is $-7177.69$ and for the SG $-7162.78$. The corresponding scores for the graph in Figure \ref{fig:HS}b are $-7245.11$ and $-7139.13$.

\begin{figure}[htb]
\begin{center}
\includegraphics[width=\textwidth]{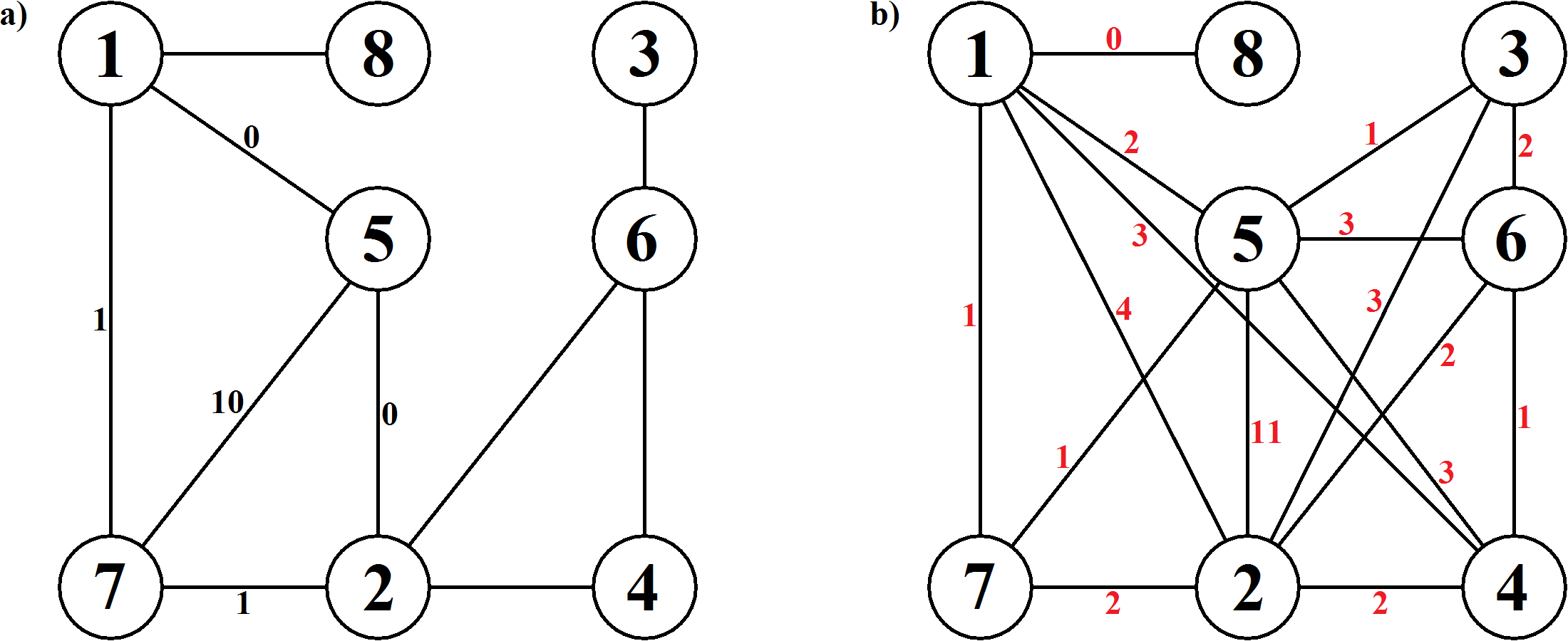}
\end{center}
\caption{Optimal SGs for parliament election data. In a) the optimal ordinary graph is amended with optimal strata. In b) the globally optimal SG with number of instances in each strata listed beside the corresponding edge.}
\label{fig:HS}
\end{figure}

These examples demonstrate that when using Markov networks, variables that would be considered conditionally dependent may in fact be independent in certain contexts. The examples also show that the globally optimal SG contains more edges than the optimal ordinary graph. This can be accredited to the fact that when using dense graphs the set of available parameter restrictions grows, while adding strata to a dense graph can still result in models that induce distributions with few free parameters. A possible method to avoid optimal SGs being very dense, and thus hampering interpretability, is to apply a stronger prior over the model space, further penalizing dense graphs or graphs with many strata as done in \citet{Pensar14}. In conclusion, these experimental results show that context-specific independencies occur naturally in various datasets and therefore it can be very useful to use graphical models that are able to capture such dependence structures.

\section{Discussion}

Graphical models, and log-linear models more generally, are useful for many types of multivariate analysis due to their interpretability. The context-specific graphical log-linear models discussed here extend the expressiveness of the stratified models considered earlier in \citet{Nyman14a} by removing the restriction concerning overlap of strata. By applying the general estimation theory developed in \cite{Rudas98} and \cite{Csiszar75}, we were able to derive a consistent procedure for estimating the parameters of a context-specific graphical log-linear model based on cyclical projections each corresponding to a specific independence restriction. Two examples with real data illustrated how the relaxation of the model class properties enables additional discovery of context-specific independencies. In future research, it would be interesting to attempt to identify  further classes of non-hierarchical restrictions to log-linear parameters, such that interpretability is maintained in the same fashion as for the current context-specific models.

\section*{Acknowledgements}
H.N. and J.P. were supported by the Foundation of \AA bo Akademi University, as part of the grant for the Center of Excellence in Optimization and Systems Engineering. J.C. was supported by the ERC grant no. 239784 and academy of Finland grant no. 251170. T.K. was supported by a grant from the Swedish research council VR/NT.

\begin{appendix}

\section*{Appendix A}
Derivation of the parameters in equation \eqref{eq:theta}.
\newline\newline
We will here give a more detailed explanation of how $\hat{\theta}_{0,0} = \hat{P}(X_{L_{\{\delta,\gamma\}}} = x_{L_{\{\delta,\gamma\}}}, X_{\Omega} = x_{\Omega}, X_{\delta} = 0, X_{\gamma} = 0)$ is derived. It is generally possible to use the factorization
\begin{gather*}
P(X_{L_{\{\delta,\gamma\}}} = x_{L_{\{\delta,\gamma\}}}, X_{\Omega} = x_{\Omega}, X_{\delta} = 0, X_{\gamma} = 0) = \\
P(X_{L_{\{\delta,\gamma\}}} = x_{L_{\{\delta,\gamma\}}}, X_{\Omega} = x_{\Omega}) P(X_{\delta} = 0, X_{\gamma} = 0 \mid X_{L_{\{\delta, \gamma\}}} = x_{L_{\{\delta,\gamma\}}}, X_{\Omega} = x_{\Omega}).
\end{gather*}
When considering a probability distribution where $\delta$ and $\gamma$ can be dependent, it is generally not true that $P(X_{\delta}, X_{\gamma}) = P(X_{\delta}) P(X_{\gamma})$. A standard result, see e.g. \cite{Whittaker90}, states that for a distribution where two variables are dependent the ML projection to the set of distributions where the variables are independent is obtained by calculating the product of the marginal probabilities of the two variables. This implies, in our case, creating a new distribution $\hat{P}$ according to
\begin{gather*}
\hat{P}(X_{L_{\{\delta,\gamma\}}} = x_{L_{\{\delta,\gamma\}}}, X_{\Omega} = x_{\Omega}, X_{\delta} = 0, X_{\gamma} = 0) = \\
P(X_{L_{\{\delta,\gamma\}}} = x_{L_{\{\delta,\gamma\}}}, X_{\Omega} = x_{\Omega})
P(X_{\delta} = 0 \mid X_{L_{\{\delta,\gamma\}}} = x_{L_{\{\delta, \gamma\}}}, X_{\Omega} = x_{\Omega}) \\
P(X_{\gamma} = 0 \mid X_{L_{\{\delta,\gamma\}}} = x_{L_{\{\delta, \gamma\}}}, X_{\Omega} = x_{\Omega}).
\end{gather*}
Using the earlier introduced notations this corresponds to setting
\begin{gather*}
\hat{\theta}_{0,0} = \hat{P}(X_{L_{\{\delta,\gamma\}}} = x_{L_{\{\delta,\gamma\}}}, X_{\Omega} = x_{\Omega} , X_{\delta} = 0, X_{\gamma} = 0) = \\
P(X_{L_{\{\delta,\gamma\}}} = x_{L_{\{\delta,\gamma\}}}, X_{\Omega} = x_{\Omega})
P(X_{\delta} = 0 \mid X_{L_{\{\delta,\gamma\}}} = x_{L_{\{\delta, \gamma\}}}, X_{\Omega} = x_{\Omega}) \\
P(X_{\gamma} = 0 \mid X_{L_{\{\delta,\gamma\}}} = x_{L_{\{\delta,\gamma\}}}, X_{\Omega} = x_{\Omega}) = \\
(\theta_{0,0}+\theta_{0,1}+\theta_{1,0}+\theta_{1,1}) \cdot(\theta_{0,0}+\theta_{0,1}) /
(\theta_{0,0}+\theta_{0,1}+\theta_{1,0}+\theta_{1,1}) \cdot\\
(\theta_{0,0}+\theta_{1,0}) / (\theta_{0,0}+\theta_{0,1}+\theta_{1,0}+\theta_{1,1}) =\\
(\theta_{0,0}+\theta_{0,1}) \cdot(\theta_{0,0}+\theta_{1,0}) / (\theta_{0,0}+\theta
_{0,1}+\theta_{1,0}+\theta_{1,1}).
\end{gather*}
The other parameters $\hat{\theta}_{0,1}$, $\hat{\theta}_{1,0}$, and $\hat{\theta}_{1,1}$ can be derived in a similar fashion.

\section*{Appendix B}
Proposal functions used for model optimization.
\newline \newline
Using the proposal function defined in Algorithm \ref{AlgoStrata}, running a sufficient amount of iterations, we can be assured to find the optimal set of strata for any chordal graph.
\begin{algorithm}
\label{AlgoStrata} Proposal function for finding optimal strata for a chordal graph.
\end{algorithm}
Let $G$ denote the underlying graph. By $L_A$ we denote all possible instances that can be added to any stratum of $G$. If $L_A$ is empty no strata may be added to $G$ and the algorithm is terminated. $L$ denotes the current state with $L$ being empty in the starting state.
\begin{enumerate}
\item Set the candidate state $L^* = L$.
\item Perform one of the following steps.
\begin{itemize}
\item[2.1.] If $L$ is empty add a randomly chosen instance from $L_A$ to $L^*$.
\item[2.2.] Else if  $\{L_A \setminus L\}$ is empty remove a randomly chosen instance from $L^*$.
\item[2.3.] Else with probability $0.5$ add a randomly chosen instance from  $\{L_A \setminus L\}$ to $L^*$.
\item[2.4.] Else remove a randomly chosen instance from $L^*$.
\end{itemize}
\end{enumerate}

\noindent Using this proposal function the optimal set of strata can be found for any underlying graph and we can proceed to the search for the best underlying graph. The proposal function in Algorithm \ref{AlgoSGM} is used for this task.
\begin{algorithm}
\label{AlgoSGM} Proposal function used to find the optimal underlying chordal graph.
\end{algorithm}
\noindent The starting state is set to be the graph containing no edges. Let $G$ denote the current graph with $G_L = (G, L)$ being the stratified graph with underlying graph $G$ and optimal set of strata $L$.
\begin{enumerate}
\item Set the candidate state $G^* = G$. 
\item Randomly choose a pair of nodes $\delta$ and $\gamma$. If the edge $\{\delta, \gamma\}$ is present in $G^*$ remove it, otherwise add the edge $\{\delta, \gamma\}$ to $G^*$.
\item While $G^*$ is non-chordal repeat steps 1 and 2.
\end{enumerate}
\noindent The resulting candidate state $G^*$ is used along with the corresponding optimal set of strata $L^*$ to form the stratified graph $G^*_L = (G^*, L^*)$ which is used when calculating the acceptance probability according to \eqref{accept}.

\section*{Appendix C}
Questions considered in parliament election data.
\begin{enumerate}
\item Since the mid-1990's the income differences have grown rapidly in Finland. How should we react to this? \\
0 - The income differences do not need to be narrowed. \\
1 - The income differences need to be narrowed.

\item Should homosexual couples have the same rights to adopt children as heterosexual couples? \\
0 - Yes. \\
1 - No.

\item Child benefits are paid for each child under the age of 18 living in Finland, independent of the parents' income. What should be done about child benefits? \\
0 - The income of the parents should not affect the child benefits. \\
1 - Child benefits should be dependent on parents' income.

\item In Finland military service is mandatory for all men. What is your opinion on this? \\
0 - The current practice should be kept or expanded to also include women. \\
1 - The military service should be more selective or abandoned altogether.

\item Should Finland in its affairs with China and Russia more actively debate issues regarding human rights and the state of democracy in these countries? \\
0 - Yes. \\
1 - No.

\item Russia has prohibited foreigners from owning land close to the borders. In recent years, Russians have bought thousands of properties in Finland. How should Finland react to this? \\
0 - Finland should not restrict foreigners from buying property in Finland. \\
1 - Finland should restrict foreigners' rights to buy property and land in Finland.

\item During recent years municipalities have outsourced many services to privately owned companies. What is your opinion on this? \\
0 - Outsourcing should be used to an even higher extent. \\
1 - Outsourcing should be limited to the current extent or decreased.

\item Currently, a system is in place where tax income from more wealthy municipalities is transferred to less wealthy municipalities. In practice this means that municipalities in the Helsinki region transfer money to the other parts of the country. What is your opinion of this system? \\
0 - The current system is good, or even more money should be transferred. \\
1 - The Helsinki region should be allowed to keep more of its tax income.
\end{enumerate}

\end{appendix}

\bibliographystyle{henrik}
\bibliography{biblio}

\begin{thebibliography}{25}
\providecommand{\natexlab}[1]{#1}
\expandafter\ifx\csname urlstyle\endcsname\relax
  \providecommand{\doi}[1]{doi:\discretionary{}{}{}#1}\else
  \providecommand{\doi}{doi:\discretionary{}{}{}\begingroup
  \urlstyle{rm}\Url}\fi

\bibitem[{Corander(2003{\natexlab{a}})}]{Corander03b}
Corander, J.
\newblock {B}ayesian graphical model determination using decision theory.
\newblock Journal of multivariate analysis, 85:253--266 (2003{\natexlab{a}}).

\bibitem[{Corander(2003{\natexlab{b}})}]{Corander03a}
Corander, J.
\newblock Labelled graphical models.
\newblock Scandinavian Journal of Statistics, 30:493--508 (2003{\natexlab{b}}).

\bibitem[{Corander et~al.(2008)Corander, Ekdahl, and Koski}]{Corander08}
Corander, J., Ekdahl, M., and Koski, T.
\newblock Parallel interacting {MCMC} for learning of topologies of graphical
  models.
\newblock Data Mining and Knowledge Discovery, 17:431--456 (2008).

\bibitem[{Corander et~al.(2006)Corander, Gyllenberg, and Koski}]{Corander06}
Corander, J., Gyllenberg, M., and Koski, T.
\newblock Bayesian model learning based on a parallel {MCMC} strategy.
\newblock Statistics and Computing, 16:355--362 (2006).

\bibitem[{Csisz\'{a}r(1975)}]{Csiszar75}
Csisz\'{a}r, I.
\newblock {$I$}-divergence geometry of probability distributions and
  minimization problems.
\newblock The Annals of Probability, 3:146--158 (1975).

\bibitem[{Csisz{\'a}r and Mat{\'u}{\u s}(2003)}]{Csiszar03}
Csisz{\'a}r, I. and Mat{\'u}{\u s}, F.
\newblock Information projections revisited.
\newblock IEEE Transactions on Information Theory, 49:1474--1490 (2003).

\bibitem[{Dellaportas and Forster(1999)}]{Dellaportas99}
Dellaportas, P. and Forster, J.~J.
\newblock {M}arkov chain {M}onte {C}arlo model determination for hierarchical
  and graphical log-linear models.
\newblock Biometrika, 86:615--633 (1999).

\bibitem[{Edwards and Havr\'{a}nek(1985)}]{Edwards85}
Edwards, D. and Havr\'{a}nek, T.
\newblock A fast procedure for model search in multidimensional contingency
  tables.
\newblock Biometrika, 72:339--351 (1985).

\bibitem[{Eriksen(1999)}]{Eriksen99}
Eriksen, P.~S.
\newblock Context specific interaction models.
\newblock Technical report, Department of Mathematical Sciences, Aalborg
  University, Aalborg (1999).

\bibitem[{Giudici and Castelo(2003)}]{Giudici03}
Giudici, P. and Castelo, R.
\newblock Improving {M}arkov chain {M}onte {C}arlo model search for data
  mining.
\newblock Machine learning, 50:127--158 (2003).

\bibitem[{Giudici and Green(1999)}]{Giudici99}
Giudici, P. and Green, P.
\newblock Decomposable graphical {G}aussian model determination.
\newblock Biometrika, 86:785--801 (1999).

\bibitem[{Golumbic(2004)}]{Golumbic04}
Golumbic, M.~C.
\newblock Algorithmic graph theory and perfect graphs.
\newblock Amsterdam: Elsevier, 2nd edition (2004).

\bibitem[{{Helsingin Sanomat}(2011)}]{HelsinginSanomat11}
{Helsingin Sanomat}.
\newblock {HS}:n vaalikone 2011 (2011).
\newblock Visited 2014-8-19.

\bibitem[{H{\o}jsgaard(2003)}]{Hojsgaard03}
H{\o}jsgaard, S.
\newblock Split models for contingency tables.
\newblock Computational Statistics \& Data Analysis, 42:621--645 (2003).

\bibitem[{H{\o}jsgaard(2004)}]{Hojsgaard04}
H{\o}jsgaard, S.
\newblock Statistical inference in context specific interaction models for
  contingency tables.
\newblock Scandinavian Journal of Statistics, 31:143--158 (2004).

\bibitem[{Koivisto and Sood(2004)}]{Koivisto04}
Koivisto, M. and Sood, K.
\newblock Exact {B}ayesian structure discovery in {B}ayesian networks.
\newblock The Journal of Machine Learning Research, 5:549--573 (2004).

\bibitem[{Koller and Friedman(2009)}]{Koller09}
Koller, D. and Friedman, N.
\newblock Probabilistic graphical models: principles and techniques.
\newblock London: The MIT Press (2009).

\bibitem[{Lauritzen(1996)}]{Lauritzen96}
Lauritzen, S.~L.
\newblock Graphical models.
\newblock Oxford: Oxford University Press (1996).

\bibitem[{Madigan and Raftery(1994)}]{Madigan94}
Madigan, D. and Raftery, A.~E.
\newblock Model selection and accounting for model uncertainty in graphical
  models using {O}ccam's window.
\newblock Journal of the American Statistical Association, 89:1535--1546
  (1994).

\bibitem[{Nyman et~al.(2014{\natexlab{a}})Nyman, Pensar, Koski, and
  Corander}]{Nyman14a}
Nyman, H., Pensar, J., Koski, T., and Corander, J.
\newblock Stratified graphical models - context-specific independence in
  graphical models.
\newblock Bayesian Analysis (2014{\natexlab{a}}).
\newblock \doi{10.1214/14-BA882}.

\bibitem[{Nyman et~al.(2014{\natexlab{b}})Nyman, Xiong, Pensar, and
  Corander}]{Nyman14b}
Nyman, H., Xiong, J., Pensar, J., and Corander, J.
\newblock Marginal and simultaneous predictive classification using stratified
  graphical models.
\newblock arXiv:1401.8078 [stat.ML] (2014{\natexlab{b}}).

\bibitem[{Pensar et~al.(2014)Pensar, Nyman, Koski, and Corander}]{Pensar14}
Pensar, J., Nyman, H., Koski, T., and Corander, J.
\newblock Labeled directed acyclic graphs: a generalization of context-specific
  independence in directed graphical models.
\newblock Data Mining and Knowledge Discovery (2014).
\newblock \doi{10.1007/s10618-014-0355-0}.

\bibitem[{Rudas(1998)}]{Rudas98}
Rudas, T.
\newblock A new algorithm for the maximum likelihood estimation of graphical
  log-linear models.
\newblock Computational Statistics, 13:529--537 (1998).

\bibitem[{Schwarz(1978)}]{Schwarz78}
Schwarz, G.
\newblock Estimating the dimension of a model.
\newblock The Annals of Statistics, 6:461--464 (1978).

\bibitem[{Whittaker(1990)}]{Whittaker90}
Whittaker, J.
\newblock Graphical models in applied multivariate statistics.
\newblock Chichester: Wiley (1990).

\end{thebibliography}

\end{document}